%% file: example_paper.tex
\documentclass{article}

\usepackage{microtype}
\usepackage{graphicx}
\usepackage{subfig}
\usepackage{booktabs} 

\usepackage{amsmath,amsthm,amssymb}
\usepackage{tikz}

\newcommand{\T}{{\text{\tiny\sffamily\upshape\mdseries T}}}
\newcommand{\bx}{\boldsymbol{x}}

\newcommand{\bh}{\boldsymbol{h}}

\newcommand{\bW}{\boldsymbol{W}}
\newcommand{\bU}{\boldsymbol{U}}
\newcommand{\bw}{\boldsymbol{w}}

\newcommand{\bg}{\boldsymbol{g}} 
 
\newcommand{\Lnorm}{\mathcal{L}} 
 
\newcommand{\subW}{{\textbf{subW}}}

\newcommand{\edr}{\mathrm{e}}

\newtheorem{theorem}{Theorem}[section]
\newtheorem{proposition}{Proposition}[section]
\newtheorem{lemma}{Lemma}[section]

\newtheorem{corollary}{Corollary}[section]
\newtheorem{definition}{Definition}[section]
\theoremstyle{remark}
\newtheorem{remark}{Remark}[section]

\usepackage{hyperref}
\hypersetup{
	colorlinks=true,
	urlcolor=blue,
	citecolor=blue,
	linktoc=all,
	linkcolor=blue}

\usepackage[accepted]{icml2019}

\icmltitlerunning{Understanding Priors in Bayesian Neural Networks at the Unit Level}
\begin{document}
\twocolumn[
\icmltitle{Understanding Priors in Bayesian Neural Networks at the Unit Level}



\icmlsetsymbol{equal}{*}

\begin{icmlauthorlist}
\icmlauthor{Mariia Vladimirova}{one,two}
\icmlauthor{Jakob Verbeek}{one}
\icmlauthor{Pablo Mesejo}{three}
\icmlauthor{Julyan Arbel}{one}
\end{icmlauthorlist}

\icmlaffiliation{one}{Univ. Grenoble Alpes, Inria, CNRS, Grenoble INP, LJK, 38000 Grenoble, France}

\icmlaffiliation{two}{Moscow Institute of Physics and Technology, 141701 Dolgoprudny, Russia}

\icmlaffiliation{three}{Andalusian Research Institute in Data Science and Computational Intelligence (DaSCI), University of Granada, 18071 Granada, Spain}

\icmlcorrespondingauthor{Mariia Vladimirova}{mariia.vladimirova@inria.fr}

\icmlkeywords{Bayesian Neural Network, sparsity, heavy-tailed prior}

\vskip 0.3in
]



\printAffiliationsAndNotice{}  

\begin{abstract}
We investigate deep Bayesian neural networks with Gaussian weight priors and a class of ReLU-like nonlinearities. 
Bayesian neural networks with Gaussian priors are well known to induce an $\mathcal{L}^2$,  ``weight decay'', regularization.
Our results characterize a more intricate regularization effect at the level of the unit activations.
Our main result  establishes that the induced prior distribution on the units before and after activation becomes increasingly heavy-tailed with the depth of the layer. 
We show that first layer units are Gaussian, second layer units are sub-exponential, and units in deeper layers are characterized by  sub-Weibull distributions. 
Our results  provide new theoretical insight on deep Bayesian neural networks, which we corroborate  with simulation experiments. 
\end{abstract}

\section{Introduction}
Neural networks (NNs), and their deep counterparts~\citep{goodfellow2016deep}, have largely been used in many research areas such as image analysis ~\citep{Krizhevsky2012}, signal processing~\citep{Graves2013}, or reinforcement learning~\citep{silver2016mastering}, just to name a few. The impressive performance provided by such machine learning approaches has greatly motivated   research  
that aims at a better understanding the driving mechanisms
behind their effectiveness. In particular, the study of the NNs distributional properties through Bayesian analysis has recently gained much attention. 

Bayesian approaches investigate models by assuming a prior distribution on their parameters. Bayesian machine learning refers to extending standard machine learning approaches with posterior inference, a line of research pioneered by works on Bayesian neural networks~\cite{neal1992bayesian,mackay1992practical}. There is a large variety of applications, e.g. gene selection~\citep{liang2018bayesian}, and the range of models is now very broad, including e.g. Bayesian generative adversarial networks~\citep{saatci2017bayesian}. See \citet{polson2017deep} for a review. The interest of the Bayesian approach to NNs is at least twofold. First, it offers a principled approach for modeling uncertainty of the training procedure, which is a limitation of standard NNs which only provide point estimates. 
A second main asset of Bayesian models is that they represent regularized versions of their classical counterparts.  
For instance, maximum a posteriori (MAP) estimation of a Bayesian regression model with double exponential (Laplace) prior is equivalent to Lasso regression~\citep{tibshirani1996regression}, while a Gaussian prior leads to ridge regression. 
When it comes to NNs, the regularization mechanism is also well appreciated in the literature, since they traditionally suffer from overparameterization, resulting in overfitting. 

Central in the field of regularization techniques is the \textit{weight decay} penalty \citep{Krogh1991}, which is equivalent to MAP estimation of a Bayesian neural network with independent Gaussian priors on the weights.
Dropout has recently been suggested as a regularization method in which neurons are randomly turned off \cite{srivastava2014dropout}, and \citet{gal2016dropout} proved that a neural network with arbitrary depth and non-linearities, with dropout applied before every weight layer, is mathematically equivalent to an approximation to the probabilistic deep Gaussian process~\citep{damianou2013deep}, leading to the consideration of such NNs as Bayesian models.

This paper is devoted to the investigation of hidden units prior distributions in Bayesian neural networks under the assumption of independent Gaussian weights. We first describe a fully connected neural network architecture as illustrated in Figure~\ref{figure:nn_visualization}. Given an input $\bx \in \mathbb{R}^N$, the $\ell$-th hidden layer unit activations are defined as
\begin{align}\label{eq:propagation}
      \bg^{(\ell)}(\bx) = \bW^{(\ell)} \bh^{(\ell - 1)} (\bx), \quad \bh^{(\ell)} (\bx) = \phi(\bg^{(\ell)}(\bx)),
\end{align}
where $\bW^{(\ell)}$ is a weight matrix including the bias vector. 
A nonlinear activation function $\phi: \mathbb{R} \to \mathbb{R}$ is applied element-wise, which is called nonlinearity, $\bg^{(\ell)}=\bg^{(\ell)}(\bx)$ is a vector of pre-nonlinearities, and $\bh^{(\ell)}=\bh^{(\ell)}(\bx)$ is a vector of post-nonlinearities. 
When we refer to either pre- or post-nonlinearities, we will use the notation $\bU^{(\ell)}$. 

\input{NN-graph.tex}
\paragraph{Contributions.}
In this paper, we extend the theoretical understanding of feedforward fully connected NNs by studying prior distributions at the units level, under the assumption of independent and normally distributed weights. 
Our contributions are the following: 
\begin{enumerate}
\item[(i)]  As our main contribution, we prove in Theorem~\ref{theorem:sub-weibull} that under some conditions on the activation function $\phi$, a Gaussian prior on the weights induces a sub-Weibull distribution on the units (both pre- and post-nonlinearities) with optimal tail parameter $\theta = \ell/2$, see Figure~\ref{figure:nn_visualization}. The condition on $\phi$ essentially imposes that $\phi$  strikes at a linear rate to $+\infty$ or $-\infty$  for large absolute values of the argument, as ReLU does. In the case of bounded support $\phi$, like sigmoid or tanh, the units  are bounded, making them \emph{de facto} sub-Gaussian\footnote{A trivial version of our main result holds, see Remark~\ref{rem:bounded}.}

\item[(ii)] We offer an interpretation of the main result from a more elaborate regularization scheme at the level of the units in Section~\ref{section:sparsity}. 
\end{enumerate}

In the remainder of the paper, we first discuss related work, and then
present our main contributions starting with the necessary statistical background and theoretical results (i), then moving to intuitions and interpretation (ii), and ending up with the description of the experiments and the discussion of the results obtained. 
More specifically, Section~\ref{section:neural_network_notations} states our main contribution, Theorem~\ref{theorem:sub-weibull}, with a proof sketch while additional technical results are deferred to~Supplementary material. Section~\ref{section:sparsity} illustrates penalization techniques, providing an interpretation for the theorem. Section~\ref{section:experiments} describes the experiments. 
Conclusions and directions for future work are presented in Section~\ref{section:conclusion}. 

\section{Related work}
\label{subsection:related_works}

Studying the distributional behaviour of feedforward networks has been a fruitful avenue for understanding these models, as pioneered by the works of Radford Neal~\citep{neal1992bayesian,neal1996bayesian} and David MacKay~\citep{mackay1992practical}. The first results in the field addressed the limiting setting when the number of units per layer tends to infinity, also called the wide regime.  
\citet{neal1996bayesian} proved that a single hidden layer  neural network with normally distributed weights tends in distribution in the wide limit either to a Gaussian process~\citep{Rasmussen:2006aa} or to an $\alpha$-stable process, depending on how the prior variance on the weights is rescaled. 
In recent works,  \citet{matthewsgaussian}, or its extended version \citet{matthews2018gaussian}, and \citet{lee2018deep} extend the result of Neal to more-than-one-layer neural networks: when the number of hidden units grows to infinity, deep neural networks (DNNs) also tend in distribution to the Gaussian process, under the assumption of Gaussian weights for properly rescaled prior variances. For the rectified linear unit (ReLU) activation function, the Gaussian process covariance function is obtained analytically \citep{cho2009kernel}. 
For other nonlinear activation functions, \citet{lee2018deep} use a numerical approximation algorithm. This Gaussian process approximation is used for instance by \citet{hayou2019impact} for  improving neural networks training strategies. \citet{novak2019bayesian} extend the results by proving the Gaussian process limit for convolutional neural networks. 

Various distributional properties are also studied in NNs regularization methods. 
The \textit{dropout} technique~\citep{srivastava2014dropout} was  reinterpreted as a form of approximate Bayesian variational inference~\citep{kingma2015variational,gal2016dropout}. 
While \citet{gal2016dropout} built a connection between dropout and the Gaussian process, \citet{kingma2015variational} proposed a way to interpret Gaussian dropout. 
They suggested \textit{variational dropout} where each weight of a model has its individual dropout rate. \textit{Sparse variational dropout}  \cite{molchanov2017variational} extends variational dropout to all possible values of dropout rates, and leads to a sparse solution.
The approximate posterior is chosen to factorize either over rows or individual entries of the weight matrices. 
The prior usually factorizes in the same way, and the choice of the prior and its interaction with the approximating posterior family are studied by \citet{hron2018variational}.
Performing dropout can be used as a Bayesian approximation but, as noted by \citet{duvenaud2014avoiding}, it has no regularization effect on infinitely-wide hidden layers.

Recent work by \citet{bibi2018analytic} provides the expression of the first two moments of the output units of a one layer NN. Obtaining the moments is a first step towards characterizing the full  distribution. However, the methodology of \citet{bibi2018analytic} is limited to the first two moments and to single-layer NNs, while we address the problem in more generality for deep NNs.

\section{Bayesian neural networks have heavy-tailed deep units}
\label{section:neural_network_notations}
The deep learning approach uses stochastic gradient descent and error back-propagation in order to fit the network parameters $(\mathbf{W}^{(\ell)})_{1\leq \ell\leq L}$, where $\ell$ iterates over all network layers. In the Bayesian approach, the parameters are random variables described by probability distributions. 

\subsection{Assumptions on neural network}
\label{subsection:assumptions}
We assume a prior distribution on the model parameters, that are the weights~$\bW$.  In particular, let all weights (including biases) be independent and have zero-mean normal distribution
\begin{align}\label{eq:gaussian-prior}
      W_{i,j}^{(\ell)} \sim \mathcal{N}(0, \sigma^2_w), 
\end{align}
for all $1\leq \ell \leq L$, $1\leq i\leq H_{\ell-1}$ and $1\leq j \leq  H_{\ell}$, with fixed variance $\sigma^2_w$. 
Given some input $\bx$, such prior distribution induces by forward propagation~\eqref{eq:propagation} a prior distribution on the pre-nonlinearities and post-nonlinearities, whose \textit{tail properties} are the focus of this section. 
To this aim, the nonlinearity $\phi$ is required to span at least half of the real line as follows. 
We introduce an extended version of the nonlinearity assumption from \citet{matthews2018gaussian}: 
\begin{definition}[Extended envelope property for nonlinearities]
\label{def::extended_envelope_property}
  A nonlinearity $\phi: \mathbb{R} \to \mathbb{R}$ is said to obey the extended envelope property if there exist $c_1, c_2 \ge 0$, $d_1, d_2 > 0$ such that the following inequalities hold 
  \begin{equation}
  \begin{aligned}
  \label{eq:envelope_property}
    |\phi(u)| &\ge c_1 + d_1 |u| \quad \text{for all }  u \in \mathbb{R}_+ \text{ or } u \in \mathbb{R}_-,\\ 
    |\phi(u)| &\le c_2 + d_2 |u| \quad \text{for all }  u \in \mathbb{R}.
  \end{aligned}
  \end{equation}
\end{definition}
The interpretation of this property is that $\phi$ must shoot to infinity at least in one direction ($\mathbb{R}_+$ or $\mathbb{R}_-$, at least linearly (first line of~\eqref{eq:envelope_property}), and also at most linearly (second line of~\eqref{eq:envelope_property}). 
Of course, compactly supported nonlinearities such as sigmoid and tanh do not satisfy the extended envelope property but the majority of other nonlinearities do, including ReLU, ELU, PReLU, and SeLU. 

We need to recall the definition of asymptotic equivalence between numeric sequences which we use to describe characterization properties of distributions:
\begin{definition}[Asymptotic equivalence for sequences]
\label{def::asymptotic_equivalence}
  Two sequences $a_k$ and $b_k$ are called asymptotic equivalent and denoted as $a_k \asymp b_k$ if there exist constants $d > 0$ and $D > 0$ such that 
  \begin{equation}
  \label{asymptotic_equivalence}
    d \le \frac{a_k}{b_k} \le D, \quad \text{for all } k \in \mathbb{N}.
  \end{equation}
\end{definition}

The extended envelope property of a function yields the following asymptotic equivalence:
\begin{lemma}
\label{lemma:ext_property_nonlinearity}
Let a nonlinearity $\phi: \mathbb{R} \to \mathbb{R}$ obey the extended envelope property. Then for any symmetric random variable $X$ the following asymptotic equivalence
holds
  \begin{equation}
  \label{eq:extended_envelope_property}
    \|\phi(X)\|_k \asymp \|X\|_k, \quad \text{for all }\, k \ge 1, 
  \end{equation}
  where $\|X\|_k =  \left( \mathbb{E}[ |X|^k]\right)^{1/k}$ is a $k$-th norm of $X$. 
\end{lemma}
The proof can be found in the supplementary material. 


\subsection{Main theorem}
\label{section:main_theorem}
This section postulates the rigorous result with a proof sketch.  
In the supplementary material one can find proofs of intermediate lemmas.

Firstly, we define the notion of \textit{sub-Weibull} random variables~\citep{kuchibhotla2018moving,vladimirova2019subweibull}. 
\begin{definition}[Sub-Weibull random variable]
\label{def:subweibull}
   A random variable $X$ satisfying for all $x > 0$ and for some $\theta > 0$
    \begin{equation}\label{eq:sub-W-tail-def}
	    \mathbb{P}(|X| \ge x) \le a \exp\left( - x^{1 / \theta} \right),
    \end{equation}
   is called a sub-Weibull random variable with so-called tail parameter $\theta$, which is denoted by $X \sim \subW(\theta)$.
\end{definition}
Sub-Weibull distributions are characterized by tails lighter than (or equally light as) Weibull distributions; in the same way as sub-Gaussian or sub-exponential distributions correspond to distributions with tails lighter than Gaussian and exponential distributions, respectively.
Sub-Weibull distributions are parameterized by a positive tail index $\theta$ and are equivalent to sub-Gaussian for $\theta=1/2$ and sub-exponential for $\theta=1$. 
To describe a tail lower bound through some sub-Weibull distribution family, i.e. a distribution of $X$ to have the tail heavier than some sub-Weibull, we define the optimal tail parameter for that distribution as the positive parameter $\theta$ characterized by:
\begin{equation}\label{eq:optimal_moment_condition}
\|X\|_k \asymp k^\theta.
\end{equation}
Then $X$ is sub-Weibull distributed with optimal tail parameter $\theta$, in the sense that for any $\theta^\prime<\theta$, $X$ is not sub-Weibull with tail parameter $\theta^\prime$ \citep[see][for a proof]{vladimirova2019subweibull}. 

The following theorem postulates the main results. 
\begin{theorem}[Sub-Weibull units]
\label{theorem:sub-weibull}
Consider a feed-forward Bayesian neural network with Gaussian priors~\eqref{eq:gaussian-prior}  and with nonlinearity $\phi$ satisfying the extended envelope condition of Definition~\ref{def::extended_envelope_property}. 
Then conditional on the input $\bx$, the marginal prior distribution\footnote{We define the \textit{marginal prior distribution} of a unit as its distribution obtained after all other units distributions are integrated out. \textit{Marginal} is to be understood by opposition to \textit{joint}, or \textit{conditional}.} induced by forward propagation~\eqref{eq:propagation} on any unit (pre- or post-nonlinearity) of the $\ell$-th hidden layer is sub-Weibull with optimal tail parameter $\theta = \ell/2$. That is for any $1\leq \ell\leq L$, and for any $1\leq m\leq H_\ell$,
$$U_m^{(\ell)}\sim \subW(\ell/2),$$
where a $\subW$ distribution is defined in Definition~\ref{def:subweibull}, and $U_m^{(\ell)}$ is either a pre-nonlinearity $g_m^{(\ell)}$ or  a post-nonlinearity $h_m^{(\ell)}$.
\end{theorem}
\begin{proof}
The idea is to prove by induction with respect to hidden layer depth $\ell$ that pre- and post-nonlinearities satisfy the asymptotic moment equivalence 
   \[
	\|g^{(\ell)} \|_k \asymp k^{\ell/2} \text{ and } 	\|h^{(\ell)} \|_k \asymp k^{\ell/2}.
	\]
The statement of the theorem then follows by the moment characterization of optimal sub-Weibull tail coefficient in Equation~\eqref{eq:optimal_moment_condition}. 

According to Lemma~1.1 from from the supplementary material, centering does not harm tail properties, then, for simplicity, we consider zero-mean distributions $W_{i,j}^{(\ell)} \sim \mathcal{N}(0, \sigma^2_w)$.  

   {\it Base step:} Consider the distribution of the first hidden layer pre-nonlinearity $g = g^{(1)}$. Since weights $\bW_m$ follow normal distribution and $\bx$ is a feature vector, then each hidden unit $\bW^{\T}_m \bx$ follow also normal distribution 
\[
	g = \bW^{\T}_m \bx  \sim \mathcal{N} (0, \sigma^2_w\|\bx\|^2).
\]
Then, for normal zero-mean variable $g$, having variance $\sigma^2 = \sigma^2_w\|\bx\|^2$, holds the equality in sub-Gaussian property with variance proxy equals to normal distribution variance and from Lemma~1.1 in the supplementary material: 
\[
	\|g\|_k \asymp \sqrt{k}.
\]
As activation function $\phi$ obeys the extended envelope property, nonlinearity moments are asymptotically equivalent to symmetric variable moments 
\[
	\|\phi(g)\|_k \asymp \|g\|_k \asymp \sqrt{k}.
\]
It implies that first hidden layer post-nonlinearities $h$ have sub-Gaussian distribution or sub-Weibull with tail parameter $\theta=1/2$ (Definition~\ref{def:subweibull}). 

{\it Inductive step:} show that if the statement holds  for $\ell - 1$, then it also holds for $\ell$. 

Suppose the post-nonlinearity of $(\ell - 1)$-th hidden layer satisfies the moment condition. 
Hidden units satisfy the non-negative covariance theorem 
(Theorem~\ref{theorem:non-negative_covariance}):
\[
	\text{Cov} \left[ \Bigl( h^{(\ell - 1)} \Bigr)^{s}, \left(\tilde h^{(\ell - 1)} \right)^{t}\right] \ge 0, \text{ for any } s,t \in \mathbb{N}. 
\] 
Let the number of hidden units in $(\ell - 1)$-th layer equals to $H$. Then according to Lemma~2.2 from the supplementary material, under assumption of zero-mean Gaussian weights, pre-nonlinearities of $\ell$-th hidden layer $g^{(\ell)} = \sum_{i=1}^H W_{m, i}^{(\ell - 1)} h_i^{(\ell - 1)}$ also satisfy the moment condition, but with $\theta = \ell/2$ 
\[
	\|g^{(\ell)} \|_k \asymp k^{\ell/2}.
\]
From the extended envelope property (Definition~\ref{def::extended_envelope_property}) post-nonlinearities $h^{(\ell)}$ satisfy the same moment condition as pre-nonlinearities $g^{(\ell)}$. This finishes the proof.
\end{proof}

\begin{remark}\label{rem:bounded}
If the activation function $\phi$ is bounded, such as the sigmoid or tanh, then the units are bounded. 
As a result, by Hoeffding's Lemma, they have a  sub-Gaussian distribution.
\end{remark}

\begin{remark}\label{rem:normalization}
Normalization techniques, such as batch normalization~\citep{ioffe2015batch} or layer normalization~\citep{ba2016layer},  significantly reduce the training time in feed-forward neural networks. 
Normalization operations can be decomposed into a set of elementary operations. According to Proposition~1.4 from the supplementary material, elementary operations do not harm the distribution tail parameter. Therefore, normalization methods do not have an influence on tail behavior. 
\end{remark}

\subsection{Intermediate theorem} 
This section states with a proof sketch that the covariance between hidden units in the neural network is non-negative. 
\begin{theorem}[Non-negative covariance between hidden units]
\label{theorem:non-negative_covariance}
  Consider the deep neural network described in, and with the assumptions of, Theorem~\ref{theorem:sub-weibull}. The covariance between hidden units of the same layer is non-negative. Moreover, for given $\ell$-th hidden layer units $h^{(\ell)}$ and $\tilde h^{(\ell)}$, it holds
  \[
    \text{Cov} \left[ \Bigl( h^{(\ell)} \Bigr)^{s}, \left(\tilde h^{(\ell)} \right)^{t}\right] \ge 0, \text{ where } s,t \in \mathbb{N}. 
  \]  
    For first hidden layer $\ell = 1$ there is equality for all $s$ and $t$. 
\end{theorem}
\begin{proof} A more detailed proof can be found in the supplementary material in Section 3. 

  Recall the covariance definition for random variables $X$ and~$Y$
  \begin{equation}
  \label{def_covariance}
    \text{Cov} \left[ X, Y \right] = \mathbb{E}[XY] - \mathbb{E}[X]\mathbb{E}[Y].
  \end{equation}

  The proof is based on induction with respect to the hidden layer number. 
  
  In the proof let us make notation simplifications: $\bw_m^{\ell} = \bW_m^{\ell}$ and $w_{mi}^{\ell} = W_{mi}^{\ell}$ for all $m \in H_\ell$. If the index $m$ is omitted, then $\bw^{\ell}$ is some the vectors $\bw_m^{\ell}$, $w_{i}^{\ell}$ is $i$-th element of the vector $\bw_m^{\ell}$.

  \textit{1. First hidden layer.}
  Consider the first hidden layer units $h^{(1)}$ and $\tilde h^{(1)}$. The covariance between units is equal to zero and the units are Gaussian, since the weights $\bw^{(1)}$ and $\tilde \bw^{(1)}$ are from $\mathcal{N}(0, \sigma_w^2)$ and independent. Thus, the first hidden layer units are independent and its covariance~\eqref{def_covariance} is equal to 0.
  Moreover, since $h^{(1)}$ and $\tilde h^{(1)}$ are independent, then $ \Bigl( h^{(1)} \Bigr)^{s}$ and $\left(\tilde h^{(1)}\right)^{t}$ are also independent.
  
  \textit{2. Next hidden layers.}
  Assume that the $(\ell - 1)$-th hidden layer has $H_{\ell - 1}$ hidden units, where $\ell >1$. Then the  $\ell$-th hidden layer pre-nonlinearity is equal to
  \begin{equation}
  \label{pre-nonlinearity}
    g^{(\ell)} = \sum_{i = 1}^{H_{\ell - 1}} w_i^{(\ell)} h_i^{(\ell - 1)}. 
  \end{equation}
  We want to prove that the covariance~\eqref{def_covariance} between the $\ell$-th hidden layer pre-nonlinearities is non-negative.
  Let us show firstly the idea of the proof in the case $H_{\ell - 1} = 1$ and then briefly describe the proof for any finite $\ H_{\ell - 1} > 1,\ H_{\ell - 1} \in \mathbb{N}$. 

  \textit{2.1 One hidden unit.} In the case $H_{\ell - 1} = 1$, the covariance~\eqref{def_covariance} sign is the same as of the expression 
  \begin{equation}
  \label{theorem:covariance:h_hidden_units_1}
    \mathbb{E} \left[\left( h^{(\ell - 1)}\right)^{2(s_1 + t_1)}\right] \nonumber - \mathbb{E} \left[\left( h^{(\ell - 1)}\right)^{2s_1}\right]\mathbb{E} \left[\left( h^{(\ell - 1)}\right)^{2t_1}\right],
  \end{equation}
  since the weighs are zero-mean distributed, its moments are equal to zero with an odd order. 
    According to Jensen's inequality for convex function $f$, we have $\mathbb{E}[f(x_1, x_2)] \ge  f(\mathbb{E}[x_1], \mathbb{E}[x_2])$.
  Since a function $f(x_1, x_2) = x_1x_2$ is convex for $x_1 \ge 0$ and $x_2 \ge 0$, then, taking $x_1 = \left( h^{(\ell - 1)}\right)^{2s_1}$ and $x_2 =\left( h^{(\ell - 1)}\right)^{2t_1}$, we have the condition we need~\eqref{theorem:covariance:h_hidden_units_1} being satisfied. 

  \textit{2.1. $H$ hidden units.} Now let us consider the covariance between pre-nonlinearities~\eqref{pre-nonlinearity} for $H_{\ell - 1} = H > 1$. Raise the sum in the brackets to the power
  \begin{align*}
    &\Bigl(\sum_{i=1}^{H} w_i^{(\ell)} h_i^{(\ell - 1)}\Bigr)^{s} =\\
    &= \sum_{s_H = 0}^{s} C_{s}^{s_H} \left(w_H^{(\ell)} h_H^{(\ell - 1)} \right)^{s_H} \Bigl(\sum_{i=1}^{H - 1} w_i^{(\ell)} h_i^{(\ell - 1)}\Bigr)^{s - s_H}. 
   \end{align*}
And the same way for the second bracket $\Bigl(\sum_{i=1}^{H} \tilde w_i^{(\ell)} h_i^{(\ell - 1)}\Bigr)^t$.
 Notice that binomial terms will be the same in the minuend and the subtrahend terms of~\eqref{def_covariance}. 
  So the covariance in our notations can be written in the form of 
  \begin{align*}
    &\text{Cov} \left[ \Bigl(\sum_{i=1}^{H_{\ell - 1}} w_i^{(\ell)} h_i^{(\ell - 1)}\Bigr)^{s}, \Bigl(\sum_{i=1}^{H_{\ell - 1}} \tilde w_i^{(\ell)} h_i^{(\ell - 1)}\Bigr)^{t} \right] =\\
    &=  \sum  \sum C \left(\mathbb{E}\left[ A B \right] - \mathbb{E}\left[ A \right]  \mathbb{E}\left[ B \right] \right),
  \end{align*}
  where $C$-terms contain binomial coefficients, $A$-terms~---~all possible products of hidden units in $\left(g^{(\ell)}\right)^s$ and $B$-terms~---~all possible products of hidden units in $\left(\tilde g^{(\ell)}\right)^t$.
  In order for the covariance to be non-negative, it is sufficient to show that the difference $\mathbb{E}\left[ A B \right] - \mathbb{E}\left[ A \right]  \mathbb{E}\left[ B \right]$ is non-negative. 
  Since the weights are Gaussian and independent, we have the following equation, omitting the superscript for simplicity,
  \begin{equation*}
    \mathbb{E}\left[ A B \right]
    = W \tilde W \cdot \mathbb{E} \left[\prod_{i=1}^H h_i^{s_i + t_i} \right],
  \end{equation*}
  \begin{equation*}
    \mathbb{E}\left[ A \right]  \mathbb{E}\left[ B \right] = W \tilde W \cdot \mathbb{E} \left[ \prod_{i=1}^H h_i^{s_i} \right] \mathbb{E} \left[\prod_{i=1}^H h_i^{t_i} \right],
  \end{equation*}
  where $W \tilde W$ is the product of weights moments
  \begin{equation*}
  W\tilde W
  = \prod_{i=1}^H \mathbb{E} \left[w_i^{s_i}\right]\mathbb{E} \left[\tilde w_i^{t_i}\right].
  \end{equation*}
  For $W \tilde W$ not equal to zero, all the powers must be even. Now we need to prove 
  \begin{equation}
  \label{theorem:covariance:h_hidden_units}
    \mathbb{E} \left[\prod_{i=1}^{H/2} h_i^{2(s_i + t_i)} \right] \ge \mathbb{E} \left[\prod_{i=1}^{H/2} h_i^{2 s_i} \right] \mathbb{E} \left[\prod_{i=1}^{H/2} h_i^{2 t_i} \right].
  \end{equation}
  According to Jensen's inequality for convex functions, since a function $f(x_1, x_2) = x_1x_2$ is convex for $x_1 \ge 0$ and $x_2 \ge 0$, then, taking $x_1 = \prod_{i=1}^{H/2} h_i^{2s_i}$ and $x_2 = \prod_{i=1}^{H/2} h_i^{2t_i}$, the condition from~\eqref{theorem:covariance:h_hidden_units} is satisfied. 

  \textit{3. Post-nonlinearities.} 
  
	Let show the proof for the ReLU nonlinearity.

	The distribution of the $\ell$-th hidden layer pre-nonlinearity $g^{(\ell)}$ is the sum of symmetric distributions, which are products of Gaussian variables $w^{(\ell)}$ and the non-negative ReLU output, i.e. the $(\ell - 1)$-th hidden layer post-nonlinearity $h^{(\ell - 1)}$. Therefore, $g^{(\ell)}$ follows a symmetric distribution and  the following inequality 
	\begin{multline*}
		\int_{-\infty}^{+\infty} \int_{-\infty}^{+\infty} g g'\,  p(g, g') \, dg \,dg'  \ge \\
		\ge \int_{-\infty}^{+\infty} g \, p(g) \, dg \cdot \int_{-\infty}^{+\infty} g' \, p(g') \,dg'
	\end{multline*}
	implies the same inequality for a positive part 
	\begin{multline*}
		\int_{0}^{+\infty} \int_{0}^{+\infty} g g'\,  p(g, g') \, dg \,dg'  \ge \\
		\ge \int_{0}^{+\infty} g \, p(g) \, dg \cdot \int_{0}^{+\infty} g' \, p(g') \,dg'.
	\end{multline*}
	Notice that the equality above is the ReLU function output and for a symmetric distribution we have
	\begin{equation}
    \label{theorem:covariance:relu_expectation}
		\int_{0}^{+\infty} x \, p(x) \, dx = \frac12 \mathbb{E}\left[ |X| \right].
	\end{equation}

That means if the non-negative covariance is proven for pre-nonlinearities, for post-nonlinearities it is also non-negative.  We omit the proof for the other nonlinearities with the extended envelope property, since instead of precise equation~\eqref{theorem:covariance:relu_expectation}, the asymptotic equivalence for moments will be used for a positive part and for a negative part --- precise expectation expressions which depend on certain nonlinearity. 
\end{proof}
\subsection{Convolutional  neural networks}
\label{neural_network_notations}
Convolutional neural networks~\citep{fukushima1982neocognitron,lecun1998gradient} are a particular kind of neural network for processing data that has a known grid-like topology, which allows to encode certain properties into the architecture. These then make the forward function more efficient to implement and vastly reduce the amount of parameters in the neural network. Neurons in such networks are arranged in three dimensions: width, height and depth. There are three main types of layers that can be concatenated in these architectures: convolutional, pooling, and fully-connected layers (exactly as seen in standard NNs). The convolutional layer computes dot products between a region in the inputs and its weights. Therefore, each region can be considered as a particular case of a fully-connected layer. Pooling layers control overfitting and computations in deep architectures. They operate independently on every slice of the input and reduces it spatially. The most commonly functions used in pooling layers are \textit{max pooling} and \textit{average pooling}. 
\begin{proposition}
\label{prop:pooling_operations}
The  operations: 1. {\normalfont max pooling} and 2. {\normalfont averaging}  
do not modify the optimal tail parameter $\theta$ of sub-Weibull family. Consequently, the result of Theorem~\ref{theorem:sub-weibull} carries over to convolutional  neural networks.
\end{proposition}
The proof can be found in the supplementary material. 

\begin{corollary}
Consider a convolutional neural network containing  convolutional, pooling and fully-connected layers under assumptions from Section~\ref{subsection:assumptions}. Then a unit of $\ell$-th hidden layer has sub-Weibull distribution with optimal tail parameter $\theta = \ell/2$, where $\ell$ is the number of convolutional and fully-connected layers.  
\end{corollary}
\begin{proof}
Proposition~\ref{prop:pooling_operations} implies that the pooling layer keeps the tail parameter. From discussion at the beginning of the section, the result of Theorem~\ref{theorem:sub-weibull} is also applied to convolutional neural networks where the depth is considered as the number of convolutional and fully-connected layers. 
\end{proof}
\section{Regularization scheme on the units}
\label{section:sparsity}

Our main theoretical contribution, Theorem~\ref{theorem:sub-weibull}, characterizes the marginal prior distribution of the network units as follows: when the depth increases, the distribution becomes more heavy-tailed.
In this section, we provide an interpretation of the result in terms of regularization at the level of the units. To this end, we first briefly recall shrinkage and penalized estimation  methods.

\subsection{Short digest on penalized estimation}

The notion of penalized estimation is probably best illustrated on the simple linear regression model, where the aim is to improve prediction accuracy by shrinking, or even putting exactly to zero, some coefficients in the regression. Under these circumstances, inference is also more \emph{interpretable} since, by reducing the number of coefficients effectively used in the model, it is possible to grasp its salient features. Shrinking is performed  by imposing a penalty on the size of the coefficients, which is equivalent to allowing for a given budget on their size. Denote the regression parameter by $\beta \in \mathbb{R}^p$, the regression sum-of-squares by $R(\beta)$, and the penalty by $\lambda L(\beta)$, where $L$ is some norm on $\mathbb{R}^p$ and $\lambda$ some positive tuning parameter. Then, the two  formulations of the regularized problem 
\begin{gather*}
    \min_{\beta\in \mathbb{R}^p} R(\beta) +\lambda L(\beta), \text{ and} \\  
    \min_{\beta\in \mathbb{R}^p} R(\beta) \ \text{ subject to } L(\beta) \leq t,
\end{gather*}
are equivalent, with some one-to-one correspondence between $\lambda$ and $t$, and are respectively termed the \textit{penalty} and the \textit{constraint} formulation. This latter formulation provides an interesting geometrical intuition of the shrinkage mechanism: the constraint $L(\beta) \leq t$ reads as imposing a total budget of $t$ for the parameter size in terms of the norm $L$. If the ordinary least squares estimator $\hat\beta^\text{ols}$ lives in the $L$-ball with surface $L(\beta) = t$, then there is no effect on the estimation. In contrast,  when $\hat\beta^\text{ols}$ is outside the ball, then the intersection of the lowest level curve of the sum-of-squares $R(\beta)$ with the $L$-ball defines the penalized estimator.

The choice of the $L$ norm has considerable effects on the problem, as can be sensed geometrically. Consider for instance $\Lnorm^q$ norms, with $q\geq 0$. For any $q>1$, the associated $\Lnorm^q$ norm is differentiable and  contours have a round shape without sharp angles. In that case, the penalty effect is to shrink the $\beta$ coefficients towards 0. The most well-known estimator falling in this class is the \textit{ridge} regression obtained with $q=2$, see Figure~\ref{fig:shrinkage} top-left panel. 
In contrast, for any $q\in(0,1]$, the  $\Lnorm^q$ norm has some  non differentiable points along the axis coordinates, see Figure~\ref{fig:shrinkage} top-right and bottom panels. Such critical points are more likely to be hit by the level curves of the sum-of-squares $R(\beta)$, thus setting exactly to zero some of the parameters. A very successful approach in this class is the Lasso obtained with $q=1$. Note that the problem is computationally much easier in the convex situation which occurs only for $q\geq 1$. 
\begin{figure}[ht]
  \centering
  \includegraphics[height=4cm]{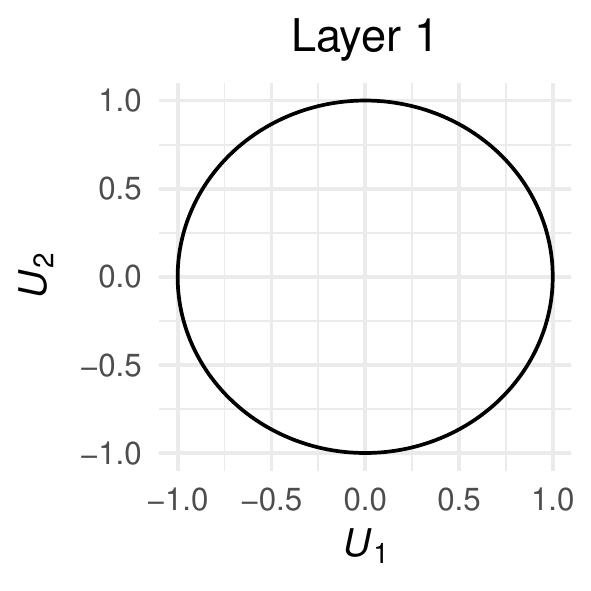}
  \includegraphics[height=4cm]{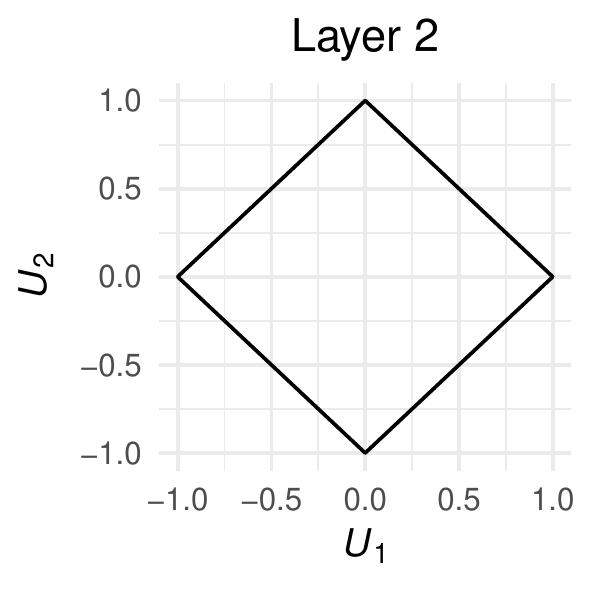}
  \includegraphics[height=4cm]{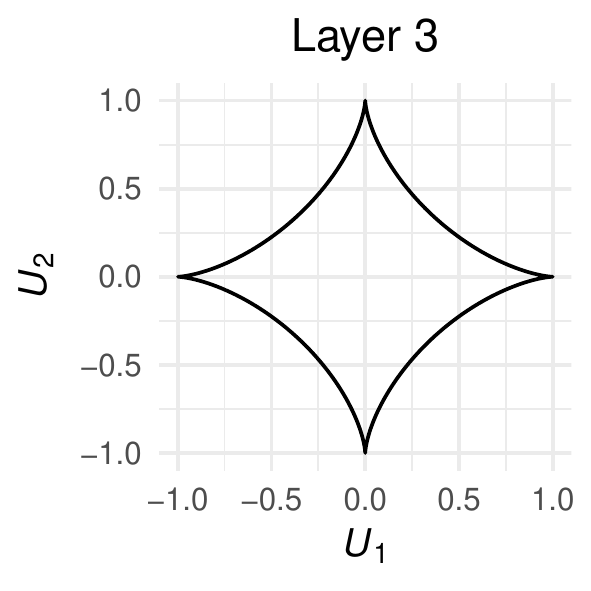}
  \includegraphics[height=4cm]{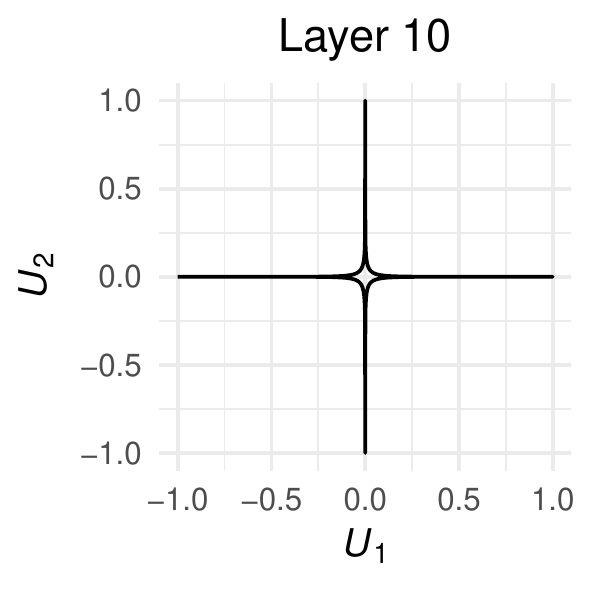}
  \caption{$\Lnorm^{2/\ell}$-norm unit balls (in dimension 2) for layers $\ell=1,2,3$ and $10$.}
  \label{fig:shrinkage}
\end{figure}

\subsection{MAP on weights $\bW$ is weight decay}

These penalized methods have a simple Bayesian counterpart in the form of the maximum a posteriori (MAP) estimator. In this context, the objective function $R$ is the negative log-likelihood, while the penalty $L$ is the negative log-prior. The objective function takes on the form of sum-of-squared errors for regression under Gaussian errors, and of cross-entropy for classification.

For neural networks, it is well-known that an independent Gaussian prior on the weights 
\begin{align}\label{eq:prod}
    \pi(\bW) \propto \prod_{\ell=1}^L\prod_{i,j}\edr^{-\frac{1}{2} (W_{i,j}^{(\ell)})^2}, 
\end{align}
is equivalent to the weight decay penalty, also known as ridge regression:
\begin{align}\label{eq:sum}
    L(\bW) = \sum_{\ell=1}^L\sum_{i,j}(W_{i,j}^{(\ell)})^2 = 
    \Vert \bW \Vert_2^2,
\end{align}
where products in~\eqref{eq:prod} and sums in~\eqref{eq:sum} involving $i$ and $j$ above are over $1\leq i\leq H_{\ell-1}$ and $1\leq j\leq H_{\ell}$, $H_0$ and $H_L$ representing respectively the input and output dimensions.

\subsection{MAP on units $\bU$}

Now moving the point of view from \textit{weights}  to \textit{units} leads to a radically different shrinkage effect. Let $U_m^{(\ell)}$ denote the $m$-th unit of the $\ell$-th layer (either pre- or post-nonlinearity). We prove in Theorem~\ref{theorem:sub-weibull} that conditional on the input $\bx$, a Gaussian prior on the weights translates into some prior on the units   $U_m^{(\ell)}$ that is marginally sub-Weibull  with optimal tail index $\theta=\ell/2$. This means that the tails of $U_m^{(\ell)}$ satisfy
\begin{align}\label{eq:tail}
          \mathbb{P}(|U_m^{(\ell)}| \ge u) \le \exp\left( - u^{2/\ell} / K_1 \right) \quad \text{for all } u \ge 0,
\end{align}
for some positive constant $K_1$. The exponent of $u$  in the exponential term above is optimal in the sense that Equation~\eqref{eq:tail} is not satisfied with some parameter $\theta^\prime$ smaller than $\ell/2$. Thus, the marginal density of $U_m^{(\ell)}$ on $\mathbb{R}$ is approximately proportional to 
\begin{equation}\label{eq:prior-density}
    \pi_m^{(\ell)}(u)\approx \edr^{-|u|^{2/\ell}/K_1}.
\end{equation}
The joint prior distribution for all the units $\bU=(U_m^{(\ell)})_{1\leq \ell \leq L, 1\leq m \leq H_\ell}$ can be expressed from all the marginal distributions by Sklar's representation theorem~\citep{sklar1959fonctions} as
\begin{equation}
\label{copula}
    \pi(\bU) = 
    \prod_{\ell=1}^L\prod_{m=1}^{H_\ell}\pi_m^{(\ell)}(U_m^{(\ell)})
    \,C(F(\bU)),
\end{equation}
where $C$ represents the copula of $\bU$ (which characterizes all the dependence between the units) while $F$ denotes its cumulative distribution function. The penalty incurred by such a prior distribution is obtained as the negative log-prior,
\begin{align}
    L(\bU) &= -\sum_{\ell=1}^L\sum_{m=1}^{H_\ell}\log\pi_m^{(\ell)}(U_m^{(\ell)})-\log C(F(\bU)),\nonumber\\
     &\overset{\text{(a)}}{\approx} \sum_{\ell=1}^L\sum_{m=1}^{H_\ell}\vert U_m^{(\ell)}\vert^{2/\ell}-\log C(F(\bU)),\nonumber\\
     &\approx \Vert \bU^{(1)}\Vert_2^{2}
     + \Vert \bU_1^{(2)}\Vert_1 +\cdots 
     + \Vert \bU^{(L)}\Vert_{2/L}^{2/L}\nonumber\\
     &\hfill\qquad-\log C(F(\bU)),\label{eq:penalty}
\end{align}
where (a) comes from~\eqref{eq:prior-density}. 
The first $L$ terms in~\eqref{eq:penalty} indicate that some shrinkage operates at every layer of the network, with a penalty term that approximately takes the form of the $\Lnorm^{2/\ell}$ norm at layer $\ell$. Thus, the deeper the layer, the stronger the regularization induced at the level of the units, as summarized in Table~\ref{table:BNN_penalty}.
\renewcommand{\arraystretch}{1.5}
\begin{table}[!ht]
\centering
\begin{tabular}{@{}cclc@{}}
\toprule
Layer                         & Penalty on $\bW$         & \multicolumn{2}{c}{Approximate penalty on $\bU$}         \\ \toprule
$1$ & $\Vert \bW^{(1)}\Vert_2^{2}$, $\Lnorm^2$   & $\Vert \bU^{(1)}\Vert_2^{2}$ & $\Lnorm^2$  (weight decay)\\\hline
$2$ & $\Vert \bW^{(2)}\Vert_2^{2}$, $\Lnorm^2$   & $\Vert \bU^{(2)}\Vert$ & $\Lnorm^1$  (Lasso)\\\hline
$\ell$ & $\Vert \bW^{(\ell)}\Vert_2^{2}$, $\Lnorm^2$   & $\Vert \bU^{(\ell)}\Vert_{2/\ell}^{2/\ell}$ & $\Lnorm^{2/\ell}$ \\ 
\bottomrule
\end{tabular}
\caption{Comparison of Bayesian neural network penalties on weights $\bW$ and units $\bU$.}
\label{table:BNN_penalty}
\end{table}
\section{Experiments}
\label{section:experiments}

We illustrate the result of Theorem~\ref{theorem:sub-weibull} on a 100 layers MLP. The hidden layers of neural network have $H_1 = 1000$, $H_2 = 990$, $H_3 = 980$, $\dots$, $H_{\ell} = 1000-10(\ell-1)$, $\dots$, $H_{100} = 10$ hidden units, respectively. The input  $\bx$ is a vector of features from $\mathbb{R}^{10^4}$.   Figure~\ref{pic:units} represents the tails of first three, 10th and 100th hidden layers pre-nonlinearity marginal distributions in logarithmic scale. 
Units of one layer have the same sub-Weibull distribution since they share the same input and prior on the corresponding weights.
The curves are obtained as histograms from a sample of size $10^5$ from the prior on the pre-nonlinearities, which is itself obtained by sampling $10^5$ sets of weights $\bW$ from the Gaussian prior~\eqref{eq:gaussian-prior} and forward propagation via~\eqref{eq:propagation}. The input vector $\bx$ is sampled with independent features from a standard normal distribution once for all at the start. The nonlinearity $\phi$ is the ReLU function.  Being a linear combination involving symmetric weights $\bW$, pre-nonlinearities $\bg$   also have a  symmetric distribution, thus we visualize only their  distribution on $\mathbb{R}_+$. 

Figure~\ref{pic:units} corroborates our main result. On the one hand, the prior distribution of the first hidden units is Gaussian (green curve), which corresponds to a $\subW(1/2)$ distribution. On the other hand, deeper layers are characterized by heavier-tailed distributions. The deepest considered layer (100th, violet curve) has an extremely flat distribution, which corresponds to a $\subW(50)$ distribution.

\begin{figure}
  \centering
  \includegraphics[height=6.2cm]{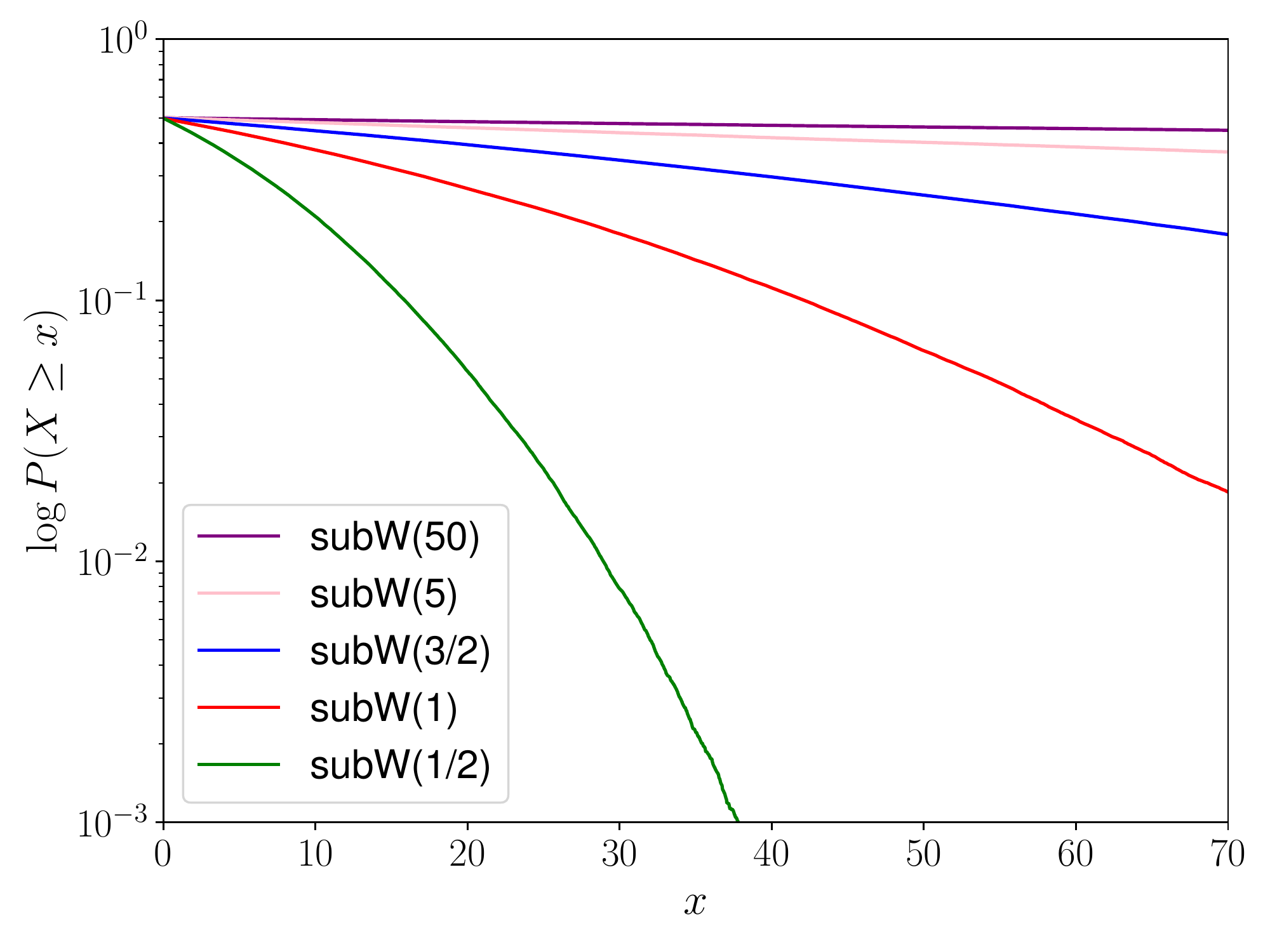}
  \caption{Illustration of layers $\ell=1,2,3,10$ and $100$ hidden units (pre-nonlinearities) marginal prior distributions. They correspond respectively to $\subW(1/2)$,  $\subW(1)$,  $\subW(3/2)$,  $\subW(5)$ and  $\subW(50)$.}
  \label{pic:units}
\end{figure}

\section{Conclusion and future work}
\label{section:conclusion}

Despite the ubiquity of deep learning  throughout science, medicine and engineering, the underlying theory has not kept pace with applications for deep learning. 
In this paper, we have extended the state of knowledge on Bayesian neural networks by providing a characterization of the marginal prior distribution of the units. 
\citet{matthews2018gaussian} and \citet{lee2018deep} proved that unit distributions have a Gaussian process limit in the wide regime, i.e.\ when the number of hidden units tends to infinity. 
We showed that they are heavier-tailed as depth increases, and discussed this result in terms of a regularizing mechanism at the level of the units. We anticipate that the  Gaussian process limit  of sub-Weibull distributions in a given layer for increasing width could be also recovered through a modification of the Central Limit Theorem for heavy-tailed distributions, see \citet{kuchibhotla2018moving}. 

Since initialization and learning dynamics are key in modern machine learning in order to
properly tune deep learning algorithms, a good implementation practice requires a proper understanding of the prior distribution at play and of the regularization it incurs. 

We hope that our results will open avenues for further research.  
Firstly, Theorem~\ref{theorem:sub-weibull} regards the \textit{marginal} prior distribution of the units, while a full characterization of the joint distribution of all units $\bU$ remains an open question. More specifically, a precise description of the copula defined in Equation~\eqref{copula} would provide valuable information about the dependence between the units, and also about the precise geometrical structure of the balls induced by that penalty. 
Secondly, the interpretation of our result (Section~\ref{section:sparsity}) is concerned with the  maximum a posteriori of the units, which is a point estimator. One of the benefits of the Bayesian approach to neural networks lies in its ability to provide a principled approach to uncertainty quantification, so that an interpretation of our result in terms of the full posterior distribution would be very appealing. Lastly,  the practical potentialities of our results are many: to 
better comprehend the regularizing mechanisms in deep neural networks will contribute to design and understand strategies to avoid overfitting and improve generalization.

\section*{Acknowledgements}

We would like to thank \href{http://mistis.inrialpes.fr/people/girard/}{St\'ephane Girard} for fruitful discussions on Weibull-like distributions and  \href{https://www.irit.fr/~Cedric.Fevotte/}{C\'edric F\'evotte} for pointing out the potential relationship of our heavy-tail result with sparsity-inducing priors.

\bibliographystyle{icml2019}

\appendix

\section{Additional technical results}


\textbf{Lemma~\ref{lemma:ext_property_nonlinearity}} proof. 
\begin{proof}
According to asymptotic equivalence definition there must exist positive constants $d$ and $D$ such that for all $k \in \mathbb{N}$ it holds
  \begin{equation}
  \label{lemma:asympt_nonlinearity}
    d  \le \|\phi(X)\|_k / \|X\|_k \le D.   
  \end{equation}
  The extended envelope property upper bound and the triangle inequality for norms imply the right-hand side of~\eqref{lemma:asympt_nonlinearity}, since 
  \[
    \|\phi(X)\|_k \le \|c_2 + d_2 |u|\|_k \le c_2 + d_2 \|u\|_k.  
  \]
  Assume that $|\phi(u)| \ge c_1 + d_1 |u|$ for $u \in \mathbb{R}_+$. Consider the lower bound of the nonlinearity moments    
  \[
    \|\phi(X)\|_k \ge \|d_1 u_+\|_k + c_1 + \| \phi(u_-)\|_k,  
  \]
  where $\{u_-:\ u\in \mathbb{R}_-\}$ and $\{u_+: \ u\in \mathbb{R}_+\}$. For negative $u_-$ there are constants $c_1 \ge 0$ and $d_1 > 0$ such that $c_1 - d_1 u > |\phi(u)|$, or $c_1 > |\phi(u)| + d_1 u$:
  \[
    \|\phi(X)\|_k > \|d_1 u_+\|_k + \||\phi(u_-)| + d_1 u_-\|_k \ge d_1 \|u\|_k. 
  \]
  It yields asymptotic equivalence~\eqref{eq:extended_envelope_property}.
\end{proof}

\textbf{Proposition~\ref{prop:pooling_operations}} proof. 
\begin{proof} 
Let $X_i \sim \subW(\theta)$ for $1\le i \le N$ be units from one region where pooling operation is applied. Using Definition~\ref{def:subweibull}, for all $x \ge 0$ and some constant $K > 0$ we have 
    $$
      \mathbb{P}(|X_i| \ge x) \le \exp\left( - x^{1/\theta} / K \right) \text{ for all } i.
    $$

Max pooling operation takes the maximum element in the region. Since $X_i$, $1\le i \le N$ are the elements in one region, we want to check if the tail of $\max_{1\le i \le N} X_i$ obeys sub-Weibull property with optimal tail parameter is equal to $\theta$. Since max pooling operation can be decomposed into linear and ReLU operations, which does not harm the distribution tail (Lemma~\ref{lemma:ext_property_nonlinearity}), it leads to the proposition statement first part. 

Summation and division by a constant does not influence the distribution tail, yielding the proposition result regarding the averaging operation.

\end{proof}
\begin{lemma}[Gaussian moments]
\label{lemma:gaussian_moments}
  Let $X$ be a normal random variable such that $X \sim \mathcal{N}(0, \sigma^2)$, then the~following asymptotic equivalence holds
  \[
    \|X\|_k \asymp \sqrt{k}.
  \]
\end{lemma}
\begin{proof}
  The moments of central normal absolute random variable $|X|$ are equal to  
  \begin{align}
  \label{formula:abs_norm_moments}
    \mathbb{E} [|X|^k] &= \int_{\mathbb{R}} |x|^k \, p(x) \, dx \nonumber \\
    &=  2 \int_0^\infty x^k \, p(x) \, dx \nonumber \\
    &= \frac1{\sqrt{\pi}}\sigma^k 2^{k/2} \Gamma \left( \frac{k + 1}2 \right).
  \end{align}
  By the Stirling approximation of the Gamma function:
  \begin{equation}
  \label{formula:stirling}
    \Gamma(z) = \sqrt{\frac{2 \pi}{z}} \left( \frac{z}{\mbox{e}}\right)^z \left(1 + O\left(\frac1z \right)\right). 
  \end{equation}
  Substituting \eqref{formula:stirling} into the central normal absolute moment \eqref{formula:abs_norm_moments}, we obtain 
  \begin{align*}
    \mathbb{E} [|X|^k] &= \frac{\sigma^k 2^{k/2}}{\sqrt{\pi}} \sqrt{\frac{4 \pi}{k + 1}} \left( \frac{k + 1}{2 \mbox{e}}\right)^{\frac{k+1}2} \left(1 + O\left(\frac1k \right)\right) \\
    &= \frac{2 \sigma^k}{\sqrt{2 \mbox{e}}} \left( \frac{k + 1}{\mbox{e}}\right)^{k/2} \left(1 + O\left(\frac1k \right)\right). 
  \end{align*}
  Then the roots of absolute moments can be written in the form of
  \begin{align}
  \label{lemma::relu_moments_equiv}
    \|X\|_k 
    &=  \frac{\sigma}{\mbox{e}^{1/(2k)}} \sqrt{ \frac{k + 1}{\mbox{e}}} \left(1 + O\left(\frac1k \right)\right)^{1/k} \nonumber \\
    &=  \frac{\sigma}{\sqrt{\mbox{e}}} \frac{\sqrt{k + 1}}{\mbox{e}^{1/(2k)} } \left(1 + O\left(\frac1{k^2} \right)\right) \nonumber\\
    &= \frac{\sigma}{\sqrt{\mbox{e}}} c_k \sqrt{k + 1} \nonumber.
  \end{align}
  Here the coefficient $c_k$ denotes
  \[
    c_k =  \frac1{\mbox{e}^{1/(2k)}} \left(1 + O\left(\frac1{k^2} \right)\right) \to 1,
  \]
  with $k \to \infty$. Thus, asymptotic equivalence holds
  \[
    \| X \|_k \asymp \sqrt{k + 1} \asymp \sqrt{k}.
  \]
\end{proof}
\begin{lemma}[Multiplication moments]
\label{lemma:multiplication_moments}
  Let $W$ and $X$ be independent random variables such that $W \sim \mathcal{N}(0, \sigma^2)$ and for some $p > 0$ it holds
  \begin{equation}
  \label{lemma:moments_condition}
    \| X \|_k \asymp k^p.
  \end{equation}
  Let $W_i$ be independent copies of $W$, and $X_i$ be copies of~$X$, $i = 1, \dots, H$ with non-negative covariance between moments of copies
  \begin{equation}
  \label{lemma:covariance_condition}
      \text{Cov} \left[X_i^s, X_j^t \right] \ge 0, \quad \text{for }i\not=j,\ s, t \in \mathbb{N}.
  \end{equation}
  Then we have the following asymptotic equivalence
  \begin{equation}
  \label{lemma:asymptotic_equivalence}
    \Bigl\| \sum_{i=1}^H W_i X_i \Bigr\|_k \asymp k^{p + 1/2}.
  \end{equation}
\end{lemma}
\begin{proof}
  Let us proof the statement, using mathematical induction. 

  {\bf Base case:} show that the statement is true for $H = 1$. For independent variables $W$ and $X$, we have
  \begin{equation}
  \label{lemma:base:one_unit_multiplication}
  \begin{split}
    \|W X \|_k &= \left(\mathbb{E} [|W X|^k ] \right)^{1/k} = \left(\mathbb{E} [|W|^k ]  \mathbb{E} [|X|^k ] \right)^{1/k} \\
    &= \|W\|_k  \|X\|_k. 
  \end{split}
  \end{equation}
  Since the random variable $W$ follows Gaussian distribution, then Lemma~\ref{lemma:gaussian_moments} implies 
  \begin{equation}
  \label{lemma:base:gaussian_norm}
  	\|W\|_k \asymp \sqrt{k}.
   \end{equation}
  Substituting assumption \eqref{lemma:moments_condition} and weight norm asymptotic equivalence \eqref{lemma:base:gaussian_norm} into \eqref{lemma:base:one_unit_multiplication} leads to the desired asymptotic equivalence~\eqref{lemma:asymptotic_equivalence} in case of $H = 1$.

  {\bf Inductive step:} show that if for $H = n - 1$ the statement holds, then for $H = n$ it also holds.

  Suppose for $H = n - 1$ we have 
  \begin{equation}
  \label{lemma:induction:assumption}
    \Bigl\| \sum_{i=1}^{n - 1} W_i X_i \Bigr\|_k \asymp k^{p + 1/2}.
  \end{equation}
  Then, according to the covariance assumption~\eqref{lemma:covariance_condition}, for $H = n$ we get
  \begin{align}
  \label{lemma:induction:lower_bound}
    \Bigl\| \sum_{i=1}^{n} W_i X_i \Bigr\|_k^k &= \Bigl\| \sum_{i=1}^{n-1} W_i X_i + W_n X_n\Bigr\|_k^k\\
    &\ge \sum_{j = 0}^k C_k^j \, \Bigl\| \sum_{i=1}^{n-1} W_i X_i \Bigr\|_{j}^{j} \Bigl\| W_n X_n \Bigr\|_{k - j}^{k - j}.
  \end{align}
  Using the equivalence definition (Def.~\ref{def::asymptotic_equivalence}), from the~induction assumption~\eqref{lemma:induction:assumption} for all $j = 0, \dots, k$ there exists absolute constant $d_1 > 0$ such that  
  \begin{equation}
  \label{lemma:induction:lower_bound_j}
    \Bigl\| \sum_{i=1}^{n-1} W_i X_i \Bigr\|_{j}^{j} \ge \Bigl(d_1 \, j^{p + 1/2}\Bigr)^j.
  \end{equation}
  Recalling previous equivalence results in the base case, there exists constant $m_2 > 0$ such that  
  \begin{equation}
  \label{lemma:induction:lower_bound_k-j}
    \Bigl\| W_n X_n \Bigr\|_{k - j}^{k - j} \ge \Bigl(d_2 \, (k - j)^{p + 1/2}\Bigr)^{k - j}.
  \end{equation}
  Substitute obtained bounds \eqref{lemma:induction:lower_bound_j} and \eqref{lemma:induction:lower_bound_k-j} into equation \eqref{lemma:induction:lower_bound} with denoted $d= \min\{d_1, d_2\}$, obtain
  \begin{multline}
  \label{theorem:covar:lower_bound}
    \Bigl\| \sum_{i=1}^{n} W_i X_i \Bigr\|_k^k \ge d^k \sum_{j = 0}^k C_k^j \, \left[j^j \, (k - j)^{k - j}\right]^{p + 1/2} \\
    = d^k \, k^{k(p + 1/2)} \sum_{j = 0}^k C_k^j \left[\Bigl(\frac{j}{k} \Bigr)^{j}  \Bigl(1 - \frac{j}k \Bigr)^{k - j} \right]^{p + 1/2}. 
  \end{multline}
  Notice the lower bound of the following expression 
  \begin{multline}
  \label{lemma:induction:bernoulli}
    \sum_{j = 0}^k C_k^j \left[\Bigl(\frac{j}{k} \Bigr)^{j}  \Bigl(1 - \frac{j}k \Bigr)^{k - j} \right]^{p + 1/2}\\
    \ge \sum_{j = 0}^k \left[\Bigl(\frac{j}{k} \Bigr)^{j}  \Bigl(1 - \frac{j}k \Bigr)^{k - j}\right]^{p + 1/2} \ge 2.
  \end{multline}
  Substituting found lower bound~\eqref{lemma:induction:bernoulli}  into~\eqref{theorem:covar:lower_bound}, get
  \begin{equation}
  \label{lemma:induction:final_lower_bound}
    \Bigl\| \sum_{i=1}^{n} W_i X_i \Bigr\|_k^k \ge 2 \,d^k \,k^{k(p + 1/2)} > d^k \,k^{k(p + 1/2)}.
  \end{equation}

  Now prove the upper bound. For random variables $Y$ and $Z$ the Holder's inequality holds 
  \begin{align*}
    \|YZ \|_1 &= \mathbb{E} \left[ |YZ| \right] \le \left( \mathbb{E}\left[ |Y|^2 \right] \mathbb{E}\left[ |Z|^{2} \right] \right)^{1/2} \\
    &= \|YZ\|_2 \|YZ\|_2.
  \end{align*}
  Holder's inequality leads to the inequality for $L^k$ norm
  \begin{equation}
  \label{lemma:induction:norm_property}
    \|YX\|_k^k \le \|Y\|_{2k}^k \|Z\|_{2k}^k.
  \end{equation}
  Obtain the upper bound of $\left\| \sum_{i=1}^{n} W_i X_i \right\|_k^k$ from the~norm property~\eqref{lemma:induction:norm_property} for the random variables $Y = \left( \sum_{i=1}^{n-1} W_i X_i \right)^{k - j} $ and $Z = \left( W_n X_n \right)^{j}$
  \begin{align}
  \label{lemma:induction:upper_bound}
    \Bigl\| \sum_{i=1}^{n} W_i X_i \Bigr\|_k^k  &= \Bigl\| \sum_{i=1}^{n-1} W_i X_i + W_n X_n \Bigr\|_k^k  \\
    &\le \sum_{j = 0}^k C_k^j \, \Bigl\| \sum_{i=1}^{n-1} W_i X_i \Bigr\|_{2j}^j \Bigl\| W_n X_n \Bigr\|_{2(k - j)}^{k - j}.
  \end{align}  
  From the induction assumption~\eqref{lemma:induction:assumption} for all $j = 0, \dots, k$ there exists absolute constant $D_1 > 0$ such that  
  \begin{equation}
  \label{lemma:induction:upper_bound_j}
    \Bigl\| \sum_{i=1}^{n-1} W_i X_i \Bigr\|_{2j}^{j} \le \Bigl(D_1 \, (2j)^{p + 1/2}\Bigr)^j.
  \end{equation}
  Recalling previous equivalence results in the base case, there exists constant $D_2 > 0$ such that  
  \begin{equation}
  \label{lemma:induction:upper_bound_k-j}
    \Bigl\| W_n X_n \Bigr\|_{2(k - j)}^{k - j} \le \Bigl(D_2 \, \bigl(2(k - j) \bigr)^{p + 1/2}\Bigr)^{k - j}.
  \end{equation}
  Substitute obtained bounds \eqref{lemma:induction:upper_bound_j} and \eqref{lemma:induction:upper_bound_k-j} into equation \eqref{lemma:induction:upper_bound} with denoted $D = \max\{D_1, D_2\}$, obtain
  \begin{align*}
    \Bigl\| \sum_{i=1}^{n} W_i X_i \Bigr\|_k^k  \le D^k \sum_{j = 0}^k C_k^j \Bigl[ \bigl(2j \bigr)^j \bigl(2(k - j)\bigr)^{k - j} \Bigr]^{p + 1/2}. 
  \end{align*}
  Find an upper bound for $\left[\left(1 - \frac{j}k \right)^{k - j} \left(\frac{j}{k} \right)^j \right]^{p + 1/2}$. Since expressions $\left(1 - \frac{j}{k} \right)$ and $\left(\frac{j}{k} \right)$ are less than 1, then $\left[\left(1 - \frac{j}k \right)^{k - j} \left(\frac{j}{k} \right)^j \right]^{p + 1/2} < 1$ holds for all natural numbers $p > 0$. For the sum of binomial coefficients it holds the inequality $\sum_{j = 0}^k C_k^j < 2^k$. So the final upper bound is 
  \begin{equation}
  \label{lemma:induction:final_upper_bound}
    \Bigl\| \sum_{i=1}^{n} W_i X_i \Bigr\|_k^k \le 2^k \, D^k \, (2k)^{k(p + 1/2)}.
  \end{equation}

  Hence, taking the $k$-th root of~\eqref{lemma:induction:final_lower_bound} and \eqref{lemma:induction:final_upper_bound}, we have upper and lower bounds which imply the equivalence for $H = n$ and the truth of inductive step 
  \[
    d'k^{p + 1/2} \le \Bigl\| \sum_{i=1}^{n} W_i X_i \Bigr\|_k  \le D'k^{p + 1/2}, 
  \]
  where $d' = d$ and $D' = 2^{p+ 3/2} D$.
  Since both the base case and the inductive step have been performed, by mathematical induction the equivalence holds for all $H \in \mathbb{N}$
  \[
    \Bigl\| \sum_{i=1}^{H} W_i X_i \Bigr\|_k \asymp k^{p + 1/2}.
  \]  
\end{proof}

\end{document}

%% file: NN-graph.tex
\tikzset{%
  input/.style={
      circle,
      draw,
      fill=blue!30!,
      minimum size=0.4cm
    },
    every neuron/.style={
      circle,
      draw,
      fill=green!30!,
      minimum size=0.4cm
    },
    neuron missing/.style={
      draw=none, 
      fill=none,
      scale=1,
      text height=0.333cm,
      execute at begin node=\color{black}$\vdots$
    },
}
\begin{figure}
\begin{center}
\begin{tikzpicture}[x=0.8cm, y=0.7cm, >=stealth]

\foreach \m/\l [count=\y] in {1,2,3,missing,4}
  \node [input/.try, neuron \m/.try] (input-\m) at (0,2.2-\y) {};

\foreach \m [count=\y] in {1,2,missing,3}
  \node [every neuron/.try, neuron \m/.try ] (firsthidden-\m) at (2,1.9-\y*1.1) {};

\foreach \m [count=\y] in {1,2,missing,3}
  \node [every neuron/.try, neuron \m/.try ] (secondhidden-\m) at (4,1.9-\y*1.1) {};

\foreach \m [count=\y] in {1,2,missing,3}
  \node [every neuron/.try, neuron \m/.try ] (thirdhidden-\m) at (6,1.9-\y*1.1) {};

\foreach \m [count=\y] in {1,2,4,3}
  \node [neuron missing/.try, neuron \m/.try ] 
(lasthidden-\m) at (8,1.9-\y*1.1) {};

\foreach \l [count=\i] in {1,2,3,n}
  \draw [<-] (input-\i) -- ++(-1,0);
  \node [above] at (input-1.north) {$\bx$};

\foreach \l [count=\i] in {1};
\node [above] at (firsthidden-1.north) {$\bh^{(1)}$};


\foreach \l [count=\i] in {1};
  \node [above] at (secondhidden-1.north) {$\bh^{(2)}$};

\foreach \l [count=\i] in {1}
  \node [above] at (thirdhidden-1.north) {$\bh^{(3)}$};

\foreach \l [count=\i] in {1,2,n}
  \draw [->] (lasthidden-\i) -- ++(1,0);
 \node [above] at (lasthidden-1.north) {$\bh^{(\ell)}$};

\foreach \i in {1,...,4}
  \foreach \j in {1,...,3}
    \draw [->] (input-\i) -- (firsthidden-\j);

\foreach \i in {1,...,3}
  \foreach \j in {1,...,3}
    \draw [->] (firsthidden-\i) -- (secondhidden-\j);

\foreach \i in {1,...,3}
  \foreach \j in {1,...,3}
    \draw [->] (secondhidden-\i) -- (thirdhidden-\j);

\foreach \i in {1,...,3}
  \foreach \j in {1,...,3}
    \draw [->] (thirdhidden-\i) -- (lasthidden-\j);

\foreach \l [count=\x from 0] in {input, $1^{\text{st}}$ hid., $2^{\text{nd}}$ hid., $3^{\text{rd}}$ hid., $\ell^{\text{th}}$ hid.}
  \node [align=center, above] at (\x*2,2) {\l \\ layer};

\foreach \l [count=\x from 0] in { ,  subW$(\frac 1 2)$,  subW$(1)$, subW$(\frac 3 2)$, subW$(\frac{\ell}2)$}
  \node [align=center, below] at (\x*2,-3) {\l};

\end{tikzpicture}
\end{center}
\caption{Neural network architecture and characterization of the $\ell$-layer units prior distribution as sub-Weibull distribution with tail parameter $\ell/2$, see Definition~\ref{def:subweibull}.}
\label{figure:nn_visualization}
\end{figure}

%% file: example_paper.bbl
\begin{thebibliography}{34}
\providecommand{\natexlab}[1]{#1}
\providecommand{\url}[1]{\texttt{#1}}
\expandafter\ifx\csname urlstyle\endcsname\relax
  \providecommand{\doi}[1]{doi: #1}\else
  \providecommand{\doi}{doi: \begingroup \urlstyle{rm}\Url}\fi

\bibitem[Ba et~al.(2016)Ba, Kiros, and Hinton]{ba2016layer}
Ba, J., Kiros, J., and Hinton, G.
\newblock Layer normalization.
\newblock \emph{arXiv preprint arXiv:1607.06450}, 2016.

\bibitem[Bibi et~al.(2018)Bibi, Alfadly, and Ghanem]{bibi2018analytic}
Bibi, A., Alfadly, M., and Ghanem, B.
\newblock Analytic expressions for probabilistic moments of {PL-DNN} with
  {G}aussian input.
\newblock In \emph{CVPR}, 2018.

\bibitem[Cho \& Saul(2009)Cho and Saul]{cho2009kernel}
Cho, Y. and Saul, L.
\newblock Kernel methods for deep learning.
\newblock In \emph{NeurIPS}, 2009.

\bibitem[Damianou \& Lawrence(2013)Damianou and Lawrence]{damianou2013deep}
Damianou, A. and Lawrence, N.
\newblock Deep {G}aussian processes.
\newblock In \emph{AISTATS}, 2013.

\bibitem[Duvenaud et~al.(2014)Duvenaud, Rippel, Adams, and
  Ghahramani]{duvenaud2014avoiding}
Duvenaud, D., Rippel, O., Adams, R., and Ghahramani, Z.
\newblock Avoiding pathologies in very deep networks.
\newblock In \emph{AISTATS}, 2014.

\bibitem[Fukushima \& Miyake(1982)Fukushima and
  Miyake]{fukushima1982neocognitron}
Fukushima, K. and Miyake, S.
\newblock Neocognitron: A self-organizing neural network model for a mechanism
  of visual pattern recognition.
\newblock In \emph{Competition and Cooperation in Neural Nets}. Springer, 1982.

\bibitem[Gal \& Ghahramani(2016)Gal and Ghahramani]{gal2016dropout}
Gal, Y. and Ghahramani, Z.
\newblock Dropout as a {B}ayesian approximation: Representing model uncertainty
  in deep learning.
\newblock In \emph{ICML}, 2016.

\bibitem[Goodfellow et~al.(2016)Goodfellow, Bengio, and
  Courville]{goodfellow2016deep}
Goodfellow, I., Bengio, Y., and Courville, A.
\newblock \emph{Deep learning}.
\newblock MIT press, 2016.

\bibitem[Graves et~al.(2013)Graves, Mohamed, and Hinton]{Graves2013}
Graves, A., Mohamed, A., and Hinton, G.
\newblock Speech recognition with deep recurrent neural networks.
\newblock In \emph{ICASSP}, pp.\  6645--6649, 2013.

\bibitem[Hayou et~al.(2019)Hayou, Doucet, and Rousseau]{hayou2019impact}
Hayou, S., Doucet, A., and Rousseau, J.
\newblock On the impact of the activation function on deep neural networks
  training.
\newblock \emph{arXiv preprint arXiv:1902.06853}, 2019.

\bibitem[Hron et~al.(2018)Hron, Matthews, and Ghahramani]{hron2018variational}
Hron, J., Matthews, A., and Ghahramani, Z.
\newblock Variational {B}ayesian dropout: pitfalls and fixes.
\newblock In \emph{ICML}, 2018.

\bibitem[Ioffe \& Szegedy(2015)Ioffe and Szegedy]{ioffe2015batch}
Ioffe, S. and Szegedy, C.
\newblock Batch normalization: Accelerating deep network training by reducing
  internal covariate shift.
\newblock In \emph{ICML}, 2015.

\bibitem[Kingma et~al.(2015)Kingma, Salimans, and
  Welling]{kingma2015variational}
Kingma, D., Salimans, T., and Welling, M.
\newblock Variational dropout and the local reparameterization trick.
\newblock In \emph{NeurIPS}, 2015.

\bibitem[Krizhevsky et~al.(2012)Krizhevsky, Sutskever, and
  Hinton]{Krizhevsky2012}
Krizhevsky, A., Sutskever, I., and Hinton, G.
\newblock {ImageNet} classification with deep convolutional neural networks.
\newblock In \emph{NeurIPS}, 2012.

\bibitem[Krogh \& Hertz(1991)Krogh and Hertz]{Krogh1991}
Krogh, A. and Hertz, J.
\newblock A simple weight decay can improve generalization.
\newblock In \emph{NeurIPS}, 1991.

\bibitem[Kuchibhotla \& Chakrabortty(2018)Kuchibhotla and
  Chakrabortty]{kuchibhotla2018moving}
Kuchibhotla, A.~K. and Chakrabortty, A.
\newblock Moving beyond sub-{G}aussianity in high-dimensional statistics:
  Applications in covariance estimation and linear regression.
\newblock \emph{arXiv preprint arXiv:1804.02605}, 2018.

\bibitem[LeCun et~al.(1998)LeCun, Bottou, Bengio, and
  Haffner]{lecun1998gradient}
LeCun, Y., Bottou, L., Bengio, Y., and Haffner, P.
\newblock Gradient-based learning applied to document recognition.
\newblock \emph{Proceedings of the IEEE}, 86\penalty0 (11):\penalty0
  2278--2324, 1998.

\bibitem[Lee et~al.(2018)Lee, Sohl-Dickstein, Pennington, Novak, Schoenholz,
  and Bahri]{lee2018deep}
Lee, J., Sohl-Dickstein, J., Pennington, J., Novak, R., Schoenholz, S., and
  Bahri, Y.
\newblock Deep neural networks as {G}aussian processes.
\newblock In \emph{ICML}, 2018.

\bibitem[Liang et~al.(2018)Liang, Li, and Zhou]{liang2018bayesian}
Liang, F., Li, Q., and Zhou, L.
\newblock Bayesian neural networks for selection of drug sensitive genes.
\newblock \emph{Journal of the American Statistical Association}, 113\penalty0
  (523):\penalty0 955--972, 2018.

\bibitem[MacKay(1992)]{mackay1992practical}
MacKay, D.
\newblock A practical {B}ayesian framework for backpropagation networks.
\newblock \emph{Neural Computation}, 4\penalty0 (3):\penalty0 448--472, 1992.

\bibitem[Matthews et~al.(2018{\natexlab{a}})Matthews, Rowland, Hron, Turner,
  and Ghahramani]{matthews2018gaussian}
Matthews, A., Rowland, M., Hron, J., Turner, R., and Ghahramani, Z.
\newblock Gaussian process behaviour in wide deep neural networks.
\newblock \emph{arXiv}, 1804.11271, 2018{\natexlab{a}}.

\bibitem[Matthews et~al.(2018{\natexlab{b}})Matthews, Rowland, Hron, Turner,
  and Ghahramani]{matthewsgaussian}
Matthews, A., Rowland, M., Hron, J., Turner, R., and Ghahramani, Z.
\newblock Gaussian process behaviour in wide deep neural networks.
\newblock In \emph{ICLR}, 2018{\natexlab{b}}.

\bibitem[Molchanov et~al.(2017)Molchanov, Ashukha, and
  Vetrov]{molchanov2017variational}
Molchanov, D., Ashukha, A., and Vetrov, D.
\newblock Variational dropout sparsifies deep neural networks.
\newblock In \emph{ICML}, 2017.

\bibitem[Neal(1992)]{neal1992bayesian}
Neal, R.
\newblock Bayesian training of backpropagation networks by the hybrid {M}onte
  {C}arlo method.
\newblock Technical Report CRG-TR-92-1, University of Toronto, 1992.

\bibitem[Neal(1996)]{neal1996bayesian}
Neal, R.
\newblock \emph{Bayesian learning for neural networks}.
\newblock Springer, 1996.

\bibitem[Novak et~al.(2019)Novak, Xiao, Bahri, Lee, Yang, Hron, Abolafia,
  Pennington, and Sohl-Dickstein]{novak2019bayesian}
Novak, R., Xiao, L., Bahri, Y., Lee, J., Yang, G., Hron, J., Abolafia, D.,
  Pennington, J., and Sohl-Dickstein, J.
\newblock Bayesian deep convolutional networks with many channels are
  {G}aussian processes.
\newblock In \emph{ICLR}, 2019.

\bibitem[Polson \& Sokolov(2017)Polson and Sokolov]{polson2017deep}
Polson, N.~G. and Sokolov, V.
\newblock Deep learning: A {B}ayesian perspective.
\newblock \emph{Bayesian Analysis}, 12\penalty0 (4):\penalty0 1275--1304, 2017.

\bibitem[Rasmussen \& Williams(2006)Rasmussen and Williams]{Rasmussen:2006aa}
Rasmussen, C. and Williams, C.
\newblock \emph{Gaussian Processes for Machine Learning}.
\newblock MIT Press, 2006.

\bibitem[Saatci \& Wilson(2017)Saatci and Wilson]{saatci2017bayesian}
Saatci, Y. and Wilson, A.
\newblock Bayesian {GAN}.
\newblock In \emph{NeurIPS}, 2017.

\bibitem[Silver et~al.(2016)Silver, Huang, Maddison, Guez, Sifre, {van den
  Driessche}, Schrittwieser, Antonoglou, Panneershelvam, and
  Lanctot]{silver2016mastering}
Silver, D., Huang, A., Maddison, C., Guez, A., Sifre, L., {van den Driessche},
  G., Schrittwieser, J., Antonoglou, I., Panneershelvam, V., and Lanctot, M.
\newblock Mastering the game of go with deep neural networks and tree search.
\newblock \emph{Nature}, 529\penalty0 (7587):\penalty0 484--489, 2016.

\bibitem[Sklar(1959)]{sklar1959fonctions}
Sklar, M.
\newblock Fonctions de repartition an dimensions et leurs marges.
\newblock \emph{Publ. inst. statist. univ. Paris}, 8:\penalty0 229--231, 1959.

\bibitem[Srivastava et~al.(2014)Srivastava, Hinton, Krizhevsky, Sutskever, and
  Salakhutdinov]{srivastava2014dropout}
Srivastava, N., Hinton, G., Krizhevsky, A., Sutskever, I., and Salakhutdinov,
  R.
\newblock Dropout: A simple way to prevent neural networks from overfitting.
\newblock \emph{Journal of Machine Learning Research}, 15\penalty0
  (1):\penalty0 1929--1958, 2014.

\bibitem[Tibshirani(1996)]{tibshirani1996regression}
Tibshirani, R.
\newblock Regression shrinkage and selection via the {L}asso.
\newblock \emph{Journal of the Royal Statistical Society. Series B
  (Methodological)}, pp.\  267--288, 1996.

\bibitem[Vladimirova \& Arbel(2019)Vladimirova and
  Arbel]{vladimirova2019subweibull}
Vladimirova, M. and Arbel, J.
\newblock {Sub-Weibull distributions: generalizing sub-Gaussian and
  sub-Exponential properties to heavier-tailed distributions}.
\newblock \emph{Preprint}, 2019.
\end{thebibliography}
